\documentclass[journal]{IEEEtran}
\IEEEoverridecommandlockouts
\usepackage{setspace}
\usepackage{cite}
\usepackage{amsmath,amssymb,amsfonts,mathtools,bm}
\usepackage{amsthm}
\usepackage{algorithm}
\usepackage{graphicx}
\usepackage{textcomp}
\usepackage{xcolor,float}
\usepackage{tabularx}
\usepackage{textcomp}
\usepackage{diagbox}
\usepackage{svg}
\usepackage{booktabs}
\usepackage{subfigure}
\usepackage{hyperref}
\newtheorem{theorem}{Theorem}

\usepackage{algpseudocode}
\usepackage{graphicx}
\usepackage{makecell}
\usepackage{amsthm}
\usepackage{caption}
\usepackage{sidecap}
\usepackage{microtype}
\usepackage{booktabs} 
\usepackage{multirow}
\usepackage{float}
\usepackage{adjustbox} 
\usepackage{amsmath}
\usepackage{amssymb}
\usepackage{colortbl}
\usepackage{makecell}
\usepackage{float}
\usepackage{url}
\usepackage{balance}
\usepackage{bbding}
\usepackage{hyperref}
\usepackage{graphicx}
\usepackage{sidecap}
\usepackage{microtype}
\usepackage{graphicx}
\usepackage{subfigure}
\usepackage{booktabs} 
\usepackage{multirow}
\usepackage{array}
\usepackage{subcaption}
\usepackage{graphicx}
\usepackage{makecell}
\usepackage{amsmath}
\usepackage{amssymb}

\newtheorem{remark}{Remark}
\usepackage{tikz}

\allowdisplaybreaks
\renewcommand{\baselinestretch}{1}
\begin{document}

\title{Latent Diffusion Model Based Denoising Receiver for 6G Semantic Communication: From Stochastic Differential Theory to Application}
\author{
Xiucheng Wang,~\IEEEmembership{Graduate Student Member,~IEEE,}
Honggang Jia,~\IEEEmembership{Graduate Student Member,~IEEE,}
Nan Cheng,~\IEEEmembership{Senior Member,~IEEE,}

\thanks{
\par This work was supported by the National Key Research and Development Program of China (2024YFB907500).
\par Xiucheng Wang, Honggang Jia, and Nan Chengare with the State Key Laboratory of ISN and School of Telecommunications Engineering, Xidian University, Xi’an 710071, China (e-mail: \{xcwang\_1, jiahg\}@stu.xidian.edu.cn; dr.nan.cheng@ieee.org);\textit{(Corresponding author: Nan Cheng.)}. 
}     
}

\maketitle
\IEEEdisplaynontitleabstractindextext

\IEEEpeerreviewmaketitle

\begin{abstract}
In this paper, a novel semantic communication framework empowered by generative artificial intelligence (GAI) is proposed, to enhance robustness against both channel noise and transmission data distribution shifts. A theoretical foundation is established using stochastic differential equations (SDEs), from which a closed-form mapping between any signal-to-noise ratio (SNR) and the optimal denoising timestep is derived. Moreover, to address distribution mismatch, a mathematical scaling method is introduced to align received semantic features with the GAI training distribution. Built on this theoretical foundation, a latent diffusion model (LDM)-based semantic communication framework is proposed to combine a variational autoencoder for semantic features extraction, where a pretrained diffusion model is used for denoising. The proposed system is a training-free framework that supports zero-shot generalization, and achieves superior performance under low-SNR and out-of-distribution conditions, offering a scalable and robust solution for future 6G semantic communication systems. Experimental results demonstrate that the proposed semantic communication framework achieves state-of-the-art performance in both pixel-level accuracy and semantic perceptual quality, consistently outperforming baselines across a wide range of SNRs and data distributions without any fine-tuning or post-training.
\end{abstract}


\section{Introduction}
The research of sixth-generation (6G) wireless communication has shown a growing paradigm shift from conventional bit-level transmission toward semantic-level information exchange \cite{dang2020should}. Unlike traditional systems that aim to reconstruct raw data, semantic communication focuses on extracting and transmitting the intended meaning of the data to support task execution \cite{zhang2024intellicise}. This evolution is driven by increasing demands for intelligent, efficient, and context-aware communication in emerging applications. For example, in extended reality (XR), user experience depends more on spatial semantics than pixel fidelity \cite{shen2023toward}; in industrial IoT (IIoT), only a small portion of sensor data is decision-critical \cite{sisinni2018industrial}; and digital twin systems require semantic-level synchronization between physical and virtual entities \cite{10623528}. However, current systems based on Shannon’s theory focus on reliable symbol delivery and lack the capability to assess contextual or task relevance \cite{qin2021semantic}. Therefore, both academia and industry communities are actively exploring semantic communication frameworks that transcend the Shannon paradigm by prioritizing meaningful content over syntactic precision \cite{jiang2021road}.

Recent advances in deep learning have led to the proliferation of neural network-based semantic communication frameworks, where the end-to-end system is typically implemented as an autoencoder architecture comprising a neural encoder, a noisy channel, and a neural decoder \cite{gunduz2024joint}. These systems demonstrate impressive performance in compressing and reconstructing semantic features, enabling tasks such as image transmission with drastically reduced bit rates \cite{xin2024semantic}. However, despite the promising capabilities of neural network-based semantic communication systems, they exhibit significant limitations in terms of robustness and generalization. One of the most critical challenges arises from their acute sensitivity to signal-to-noise ratio (SNR) variations \cite{islam2024deep}. Specifically, the semantic features produced by the encoder are susceptible to corruption by additive noise introduced by the wireless channel, resulting in the received representations at the decoder diverging markedly from those observed during training. This leads to significant performance degradation under low-SNR or unexpected channel conditions \cite{vapnik1998statistical}. Fundamentally, this challenge arises from a distribution mismatch between training and inference inputs, a phenomenon commonly referred to as the out-of-distribution (OOD) problem \cite{goodfellow2016deep}. The issue is further exacerbated by the intrinsic properties of latent semantic representations. Unlike explicit modalities such as images or audio, latent features are highly abstract, compactly encoded, and often reside in non-Euclidean manifolds \cite{goodfellow2016deep}. This compressed structure renders them particularly fragile to perturbations. Even slight Gaussian noise can induce significant distributional shifts that compromise the underlying semantic structure. Consequently, neural decoders struggle to accurately reconstruct the original content, as they fail to extract meaningful patterns from perturbed latent vectors \cite{vapnik1998statistical}. This fragility of the latent space poses a huge challenge to semantic communication systems, namely, how to effectively model, diagnose, and mitigate noise-induced performance degradation in compressed semantic domains to improve the fidelity and reliability of the semantic feature decoding. 

Furthermore, current neural semantic systems are often trained on data from a particular distribution, limiting their scalability to broader domains. When the distribution of transmitted data changes, even slightly, from the distribution of cats to the distribution of dogs, the performance of the system may deteriorate substantially, necessitating full retraining \cite{zhang2024intellicise}. Such dependence on narrowly scoped training regimes not only impairs scalability but also introduces a prohibitive computational burden when deployed in dynamic or heterogeneous environments. Consequently, the lack of noise resilience, data generalization, and model adaptability remains a fundamental bottleneck that hinders the widespread application of neural network-based semantic communication in practical 6G scenarios \cite{kim2018communication}.

The above limitations arise from a fundamental mismatch between the nature of semantic communication and the predominantly discriminative architectures used in current neural solutions \cite{goodfellow2016deep}. Most systems adopt auto-encoder backbones, training neural networks to extract semantic features under matched training and inference conditions. Although these models perform adequately on in-distribution data, they deteriorate sharply when the channel is highly noisy or when test statistics deviate from the training distribution \cite{bourtsoulatze2019deep}. From an information-theoretic perspective, the joint source–channel coding should leverage the data’s prior distribution to achieve superior compression and robustness \cite{gunduz2008joint}. However, discriminative networks are ill-equipped to learn full data priors. This motivates a shift toward generative modelling, which, in principle, captures the underlying distribution and enables maximum a posteriori recovery from corrupted latent features. Recent attempts to incorporate generative models have relied on conditional diffusion or GAN architectures \cite{croitoru2023diffusion,xin2024semantic}, yet these approaches inject artificial Gaussian noise during training, which creates a mismatch with real channels, and still depend on discriminative conditional encoders, leaving them vulnerable to OOD inputs. Consequently, a fully generative framework that eliminates conditional bottlenecks and directly exploits learned priors is essential for robust denoising and semantic reconstruction across diverse channel conditions.

To address the aforementioned challenges, in this paper, we propose a novel semantic communication architecture based on latent diffusion models (LDMs) \cite{LDM}. The core idea is to harness the generative power of LDMs for semantic-level denoising in the latent space of a variational autoencoder (VAE) \cite{kingma2013auto}. Rather than operating in the high-dimensional signal domain, the proposed method encodes semantic information into a compact latent vector, transmits it over a noisy channel, and applies a pretrained LDM at the receiver to iteratively remove Gaussian noise via reverse diffusion. This enables robust semantic recovery under severe SNR degradation without retraining. The framework is theoretically grounded in stochastic differential equation (SDE) theory, modeling diffusion as a continuous-time Markov process \cite{song2020score,huang2024simultaneous}. A closed-form mapping between the SNR and the diffusion timestep is derived, allowing the receiver to adaptively select the number of denoising steps based on channel conditions. To mitigate the OOD issue from mismatched input statistics, a linear scaling mechanism is introduced to align received features with the training distribution of the LDM. This dual adaptation strategy ensures strong generalization across diverse noise levels. Since the LDM operates independently of the encoder–decoder pipeline and serves solely as a denoiser, pretrained models from vision tasks can be seamlessly integrated without fine-tuning. This plug-and-play capability reduces system complexity and supports continual performance enhancement as more advanced generative models emerge. The main contributions of this paper are summarized as follows.
\begin{enumerate}
    \item Based on SDE, a novel semantic communication theory is proposed, which models the relationship between channel-induced Gaussian noise and the reverse diffusion process using stochastic differential equations. A closed-form mapping between the SNR and the optimal denoising timestep is derived, enabling adaptive noise suppression without retraining, and addressing the long-standing challenge of OOD robustness in semantic communication.
    \item Built on the proposed semantic communication theory, an LDM-based framework is presented, which integrates a VAE for semantic compression, a scaling mechanism for distribution alignment, and a pretrained LDM for denoising in latent space. As the LDM operates independently of the encoder–decoder pair, it supports seamless substitution with a wide range of pretrained diffusion models, significantly simplifying system deployment and upgrade.
    \item Experiment results show the proposed semantic communication framework achieves SOTA performance in both pixel-level accuracy and semantic perceptual quality, across a wide range of SNR and transmission data distribution without any fine-tuning or post-training.
\end{enumerate}

\begin{figure*}
    \centering
    \includegraphics[width=1\linewidth]{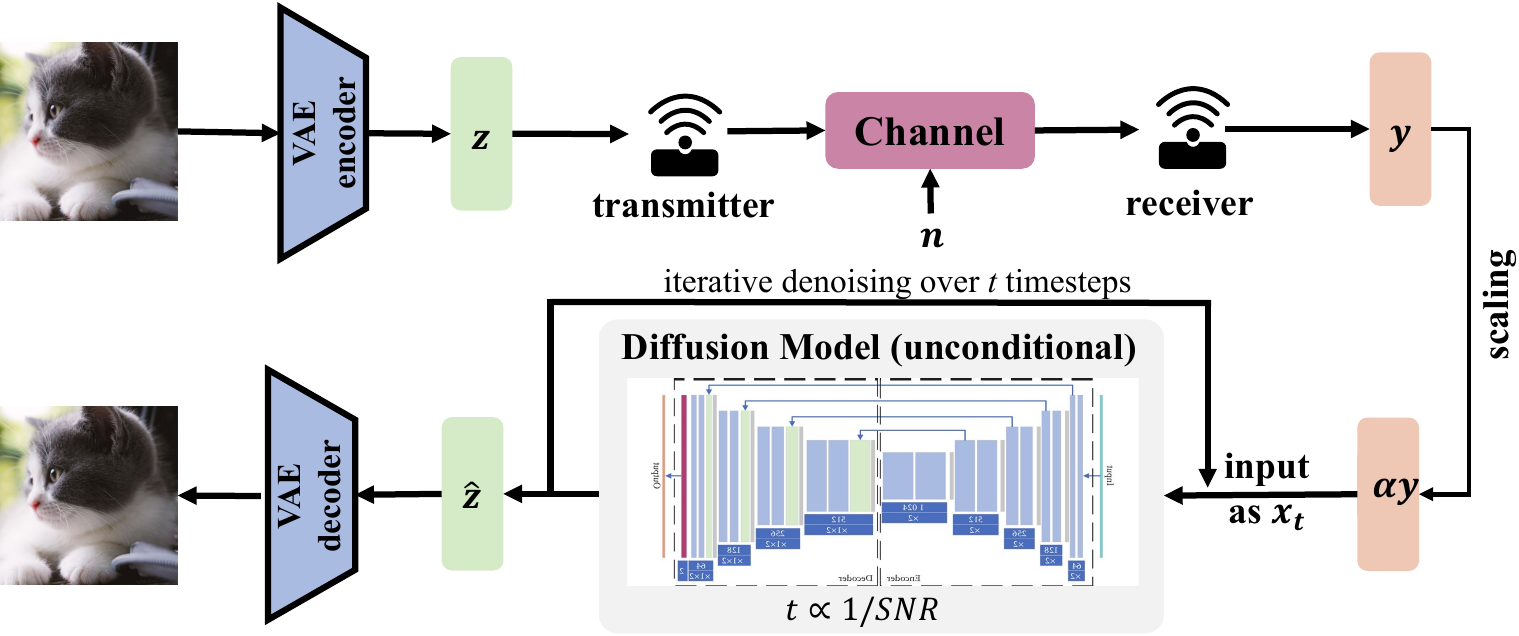}
    \caption{Illustration of the proposed semantic communication framework.}
    \label{fig-system}
\end{figure*}

\section{Preliminaries and Related Works}
\subsection{Semantic Communications}
Semantic communication has emerged as a transformative paradigm in the evolution of next-generation wireless systems, aiming to bridge the gap between signal transmission and meaning understanding. Unlike conventional Shannon-centric systems that prioritize bit-level fidelity, semantic communication seeks to effectively transmit the intended meaning of messages, thus aligning the communication process with the task-specific utility at the receiver. This paradigm shift has triggered intensive theoretical and practical exploration in recent years, resulting in a rich body of literature.

Foundational works such as \cite{qin2021semantic} and \cite{9679803} established the conceptual framework for semantic communication, delineating the limitations of Shannon theory in capturing semantic and effectiveness layers of communication, and advocating for AI-integrated approaches. Subsequent surveys, including  \cite{10319671} and \cite{xin2024semantic}, have categorized enabling technologies such as knowledge-based reasoning, deep learning-based encoders, and task-oriented transmission frameworks. These reviews have also formalized key metrics, including semantic entropy, semantic rate-distortion, and semantic channel capacity, marking a theoretical extension beyond classical information theory.

In parallel, significant efforts have been made toward system-level innovations. In \cite{nguyen2025contemporary} and \cite{wang2025semantic}, the authors review emerging architectures that integrate semantic communication with deep joint source-channel coding (JSCC), generative AI, and federated learning. These contributions underline the role of large language models (LLMs) and semantic knowledge bases (KBs) in enhancing interpretability, compressibility, and adaptability of semantic encoders across modalities such as text, image, audio, and video. The work by Lema et al. \cite{zhang2024semantic} highlights the convergence of semantic communication with edge computing and proposes scalable architectures for real-time, context-aware applications.

A number of domain-specific studies have investigated semantic communication across critical verticals. For instance, Ye et al. examine its integration into vehicular networks, focusing on multimodal semantic extraction, cooperative coding, and dynamic resource allocation \cite{ye2025survey}. Similarly, other works explore semantic transmission in metaverse platforms, Internet of Things (IoT), wireless sensor networks (WSN), and healthcare, emphasizing the ability to reduce bandwidth consumption and improve robustness through semantic compression and prediction. Despite these advances, several open challenges remain. Theoretical gaps persist in defining a unified framework for semantic information quantification and in establishing universal semantic performance metrics. Additionally, the issues of dynamic KB synchronization, adversarial robustness, interpretability of deep semantic models, and semantic-level security threats have become critical research frontiers. Recent studies, includes \cite{ma2025modeling}, have introduced performance modeling techniques, such as the Alpha-Beta-Gamma (ABG) expression, to bridge empirical deep learning performance with classical SNR-based communication analysis.

\subsection{Diffusion Model}
Diffusion models are a class of generative models that produce data samples by iteratively reversing a predefined noise-injection process. Originally proposed as an alternative to generative adversarial networks (GANs) \cite{creswell2018generative}, diffusion models have demonstrated remarkable performance in diverse generative tasks, including image synthesis, natural language modeling, and structured data generation \cite{lin2023diffbir, austin2021structured}. Beyond generation, they have also shown promise in perception-oriented applications such as image segmentation, object detection, and model-based reinforcement learning \cite{LDM,wang2024radiodiff,wang2025radiodiffinverse}. At their core, diffusion models define a two-stage process: a forward diffusion stage that incrementally perturbs input data with Gaussian noise, and a reverse denoising stage that reconstructs the data through a learned Markov process. Let $\bm{x}_0$ denote the clean input. The forward process generates latent variables $\bm{x}_1, \ldots, \bm{x}_T$ via a Markov chain with Gaussian transitions given by
\begin{align}
q\left(\bm{x}_t \mid \bm{x}_{t-1}\right) = \mathcal{N}\left(\sqrt{1 - \beta_t} \bm{x}_{t-1}, \beta_t \bm{I}\right),
\end{align}
where $\mathcal{N}(\bm{\mu},\sigma^2\bm{I})$, is the Gaussian distribution with the mean if $\bm{\mu}$ and covariance of $\sigma\bm{I}$, $\bm{I}$ is an identical matrix, and $\beta_t \in (0,1)$ controls the noise variance at step $t$. By defining $a_t = 1 - \beta_t$ and $\bar{a}_t = \prod_{s=1}^t a_s$, one obtains the closed-form marginal as follows.
\begin{align}
q(\bm{x}_t \mid \bm{x}_0) = \mathcal{N}\left(\sqrt{\bar{a}_t} \bm{x}_0, (1 - \bar{a}_t) \bm{I}\right),
\end{align}
Applying the sampling rule, the following equation can be obtained.
\begin{align}
\bm{x}_t = \sqrt{\bar{a}_t} \bm{x}_0 + \sqrt{1 - \bar{a}_t} \bm{\epsilon}, \quad \bm{\epsilon} \sim \mathcal{N}(\bm{0}, \bm{I}).
\end{align}
The reverse process approximates the true posterior using a parameterized model $p_{\bm{\theta}}(\bm{x}_{t-1} \mid \bm{x}_t)$, defined as follows.
\begin{align}
p_{\bm{\theta}}(\bm{x}_{t-1} \mid \bm{x}_t) = \mathcal{N}\left(\bm{\mu}_{\bm{\theta}}(\bm{x}_t, t), \beta_t \bm{I}\right),
\end{align}
where $\bm{\mu}_{\bm{\theta}}$ is a neural network trained to predict the denoised signal at each step. Sampling proceeds from $\bm{x}_T \sim \mathcal{N}(\bm{0}, \bm{I})$ down to $\bm{x}_0$, with each denoising step computed as follows.
\begin{align}
&\bm{x}_{t-1} = \frac{1}{\sqrt{a_t}} \left( \bm{x}_t - \frac{1 - a_t}{\sqrt{1 - \bar{\alpha}t}} \bm{\mu}_{\bm{\theta}}(\bm{x}_t, t) \right) + \beta_t \tilde{\bm{\epsilon}},\\
&\tilde{\bm{\epsilon}} \sim \mathcal{N}(\bm{0}, \bm{I})
\end{align}
where $\tilde{\bm{\epsilon}}$ is the Gaussian noise added in the reverse process. The final noise term ensures that the variance of the denoised sample aligns with that of the forward process. While a more precise formulation would scale this term using $\tilde{\beta}_t = \frac{1 - \bar{a}_{t-1}}{1 - \bar{a}_t} \beta_t$, empirical studies \cite{ho2020denoising} have shown that using $\beta_t$ directly provides a favorable trade-off between computational efficiency and denoising performance.

\section{System Model and Problem Formulation}
We consider a semantic communication system that transmits high-level representations of source data through a noisy wireless channel, with the goal of reconstructing the intended semantic content at the receiver. Unlike traditional symbol-based communication systems, the focus here lies in end-to-end semantic fidelity rather than exact bit-level recovery.

The end-to-end system comprises a source, a joint source–channel encoder, a fading channel with additive white Gaussian noise (AWGN), a denoising module, and a decoder. Let $\bm{x} \in \mathbb{R}^n$ denote the original source data. The source encoder, parameterized by $\bm{\phi}$, maps $\bm{x}$ into a compact semantic latent vector $\bm{z} \in \mathbb{R}^d$, with $d \ll n$, via
\begin{align}
\bm{z} = f_{\bm{\phi}}(\bm{x}),
\end{align}
where $f_{\bm{\phi}}: \mathbb{R}^n \to \mathbb{R}^d$ is typically realized by a variational autoencoder (VAE). The semantic latent vector $\bm{z}$ is transmitted over a fading channel with channel, yielding the received representation
\begin{align}
\tilde{\bm{z}} = \eta\bm{z} + \bm{n}, \quad \bm{n} \sim \mathcal{N}(\bm{0}, \sigma^2 \bm{I}),
\end{align}
where $\eta$ is the signal power attenuation factor caused by channel fading, and the $\sigma^2$ denotes the channel noise power\footnote{Although the analysis in this paper assumes linear channel fading with additive Gaussian noise, \cite{10607932} demonstrates that a VAE–ADMM pre-processing architecture can transform data corrupted by any non-Gaussian noise into an equivalent representation consisting of amplitude fading followed by zero-mean Gaussian noise. Consequently, the channel model adopted here remains applicable to wireless links whose noise statistics deviate from Gaussian assumptions.}.

At the receiver, the corrupted latent vector $\tilde{\bm{z}}$ undergoes denoising to mitigate the impact of channel noise. Let $g_{\bm{\psi}}: \mathbb{R}^d \to \mathbb{R}^d$ represent the denoising function parameterized by $\bm{\psi}$, and $h_{\bm{\theta}}: \mathbb{R}^d \to \mathbb{R}^n$ denote the semantic decoder parameterized by $\bm{\theta}$. The final reconstructed data $\hat{\bm{x}}$ is obtained as
\begin{align}
\hat{\bm{x}} = h_{\bm{\theta}} \left( g_{\bm{\psi}}(\tilde{\bm{z}}) \right).
\end{align}
In conventional systems, $g_{\bm{\psi}}$ may be a simple denoising autoencoder trained under a fixed SNR regime. However, such systems are inherently sensitive to variations in noise and data distribution, often resulting in suboptimal performance under mismatched conditions. In contrast, the proposed architecture adopts a pretrained LDM as $g_{\bm{\psi}}$, offering a generative and noise-adaptive alternative.

The design objective is to minimize the semantic distortion between the reconstructed output $\hat{\bm{x}}$ and the ground truth input $\bm{x}$. A natural choice of loss function is the mean squared error (MSE):
\begin{align}
\min_{\bm{\phi}, \bm{\psi}, \bm{\theta}} , \mathbb{E}_{\bm{x}, \bm{n}}\left[\left\| \bm{x} - h_{\bm{\theta}} \left( g_{\bm{\psi}}(\bm{z} + \bm{n}) \right) \right\|_{2}^{2} \right],
\end{align}
subject to the encoder constraint $\bm{z} = f_{\bm{\phi}}(\bm{x})$. This formulation captures the full stochasticity of the channel and reflects the end-to-end performance of the semantic transmission process. Two primary challenges arise in solving this problem. First, due to the dimensionality reduction $d \ll n$, the decoder must reconstruct a high-dimensional signal from a compressed and noise-contaminated latent representation. This makes the inverse mapping severely ill-posed and sensitive to noise perturbations. Second, the Gaussian noise introduced during transmission alters the distribution of the latent variable at the receiver, causing a significant mismatch between the training and inference distributions, commonly referred to as the OOD problem. These challenges underscore the need for a robust, distribution-aware denoising mechanism that can generalize across varying SNRs and input distributions. To address this, we propose a theoretically grounded denoising approach based on stochastic differential equations and generative diffusion models, which will be detailed in the next section.

\section{Diffusion Based Semantic Communication Framework}
\subsection{Theoretical Basis of Diffusion Model-Based Denoising}
DM enables high-fidelity data generation and denoising by simulating a stochastic process that gradually perturbs structured data into noise and then learns to reverse this process. While originally formulated as a discrete Markov chain, recent advances have shown that the diffusion process can also be interpreted as a continuous-time SDE, providing a principled framework for both theoretical analysis and practical acceleration.

Following the framework of \cite{song2020score}, the forward diffusion process can be equivalently represented as an Itô SDE of the form:
\begin{align}
d\bm{x}_t = f_t \bm{x}_t dt + g_t d\bm{w}_t, \label{eq:forward_sde}
\end{align}
where $\bm{x}_t \in \mathbb{R}^d$ denotes the latent variable at time $t \in [0, T]$, $\bm{w}_t$ is a standard Wiener process, $f_t$ is a time-dependent drift coefficient, and $g_t$ is the diffusion coefficient. For tractability, we set $g_t = 1$ and define the drift as $\bm{h}_t=f_t \bm{x}_t$, yielding the integral representation:
\begin{align}
\bm{x}_t = \bm{x}_0 + \int_0^t \bm{h}_s ds + \int_0^t d\bm{w}_s. \label{eq:forward_integral}
\end{align}
where the $\bm{x}_0 + \int_0^t \bm{h}_s d$ denotes the signal power attenuation process and the $\int_0^t d\bm{w}_s$ is the the noise increasing
process. Under this formulation, the conditional distribution of $\bm{x}_t$ given the initial state $\bm{x}_0 \sim q(\bm{x}_0)$ is Gaussian:
\begin{align}
q(\bm{x}_t \mid \bm{x}_0) = \mathcal{N} \left( \bm{x}_t ; \bm{x}_0 + \int_0^t \bm{h}_s ds,  t\bm{I} \right). \label{forward_marginal}
\end{align}
Alternatively, this can be expressed as follows
\begin{align}
\bm{x}_t = \bm{x}_0 + \int_0^t \bm{h}_s  ds + \sqrt{t} \bm{\epsilon}, \quad \bm{\epsilon} \sim \mathcal{N}(\bm{0}, \bm{I}).\label{raw-forward}
\end{align}

To recover the original clean latent representation $\bm{x}_0$ from the noisy variable $\bm{x}_t$, the reverse-time SDE is formulated as follows.
\begin{align}
d\bm{x}_t = \left[ f_t \bm{x}_t - \frac{g_t^2}{\beta_t} \bm{\epsilon}_t \right] dt + g_t d\bar{\bm{w}}_t, \label{reverse-sde}
\end{align}
where $\bar{\bm{w}}_t$ is an independent Wiener process and $\bm{\epsilon}_t$ represents the perturbation introduced during the forward process, scaled by the noise schedule $\beta_t$. Setting $f_t \bm{x}_t = \bm{h}_t$ and $g_t = 1$ yields a simplified expression of the reverse-time transition. In discrete time, this reverse process can be approximated as follows.
\begin{align}
\bm{x}_{t - \Delta t} = \bm{x}_t + \int_t^{t - \Delta t} \bm{h}_s ds - \frac{\Delta t}{t} \bm{\epsilon} + \sqrt{\frac{\Delta t(t - \Delta t)}{t}} \widetilde{\bm{\epsilon}}, \label{ddm-reverse}
\end{align}
where $\bm{\epsilon} \sim \mathcal{N}(\bm{0}, \bm{I})$ corresponds to the forward noise term and $\widetilde{\bm{\epsilon}} \sim \mathcal{N}(\bm{0}, \bm{I})$ is a newly added Gaussian noise for variance matching in the reverse step. These noise terms are statistically independent, reflecting the stochastic symmetry of the diffusion process.

From \eqref{ddm-reverse}, it is evident that the success of denoising hinges on the accurate estimation of the drift term $\bm{h}_t$ and the forward noise $\bm{\epsilon}$. In practice, modern denoising diffusion models train neural networks to predict either $\bm{h}_t$ directly (as in score-based generative modeling) or the noise term $\bm{\epsilon}$ added during the forward process. The predicted value is then used to guide the reverse sampling path from $\bm{x}_T$ to $\bm{x}_0$. The feasibility of this prediction is supported by the universal approximation theorem, which guarantees that a sufficiently expressive neural network can approximate the mappings required for effective denoising. State-of-the-art systems such as Stable Diffusion and DALL·E rely on this principle to reconstruct high-dimensional images from noise, demonstrating the practical effectiveness of neural diffusion-based inference.

Remarkably, since the $\bm{h}_t$ holds $\int_{0}^{1}\bm{h}_t +\bm{x}_0 = 0$, similar to \cite{huang2024decoupled} by setting the $\bm{h}_t = \bm{x}_0$, a simplified expression of \eqref{raw-forward} and \eqref{ddm-reverse} can be obtained as follows .
\begin{align}
    &\bm{x}_{t} = (1-t)\bm{x}_{0}+\sqrt{t} \bm{\epsilon}\label{ddm-forward},\\
    &\bm{x}_{t-\Delta t} = (1-\Delta t)\bm{x}_{t} - \frac{\Delta t}{t} \bm{\epsilon} + \sqrt{\frac{\Delta t(t - \Delta t)}{t}} \widetilde{\bm{\epsilon}}.
\end{align}


\subsection{Relationship Between Diffusion Variable and Noisy Receiving}\begin{figure}
    \centering
    \includegraphics[width=1\linewidth]{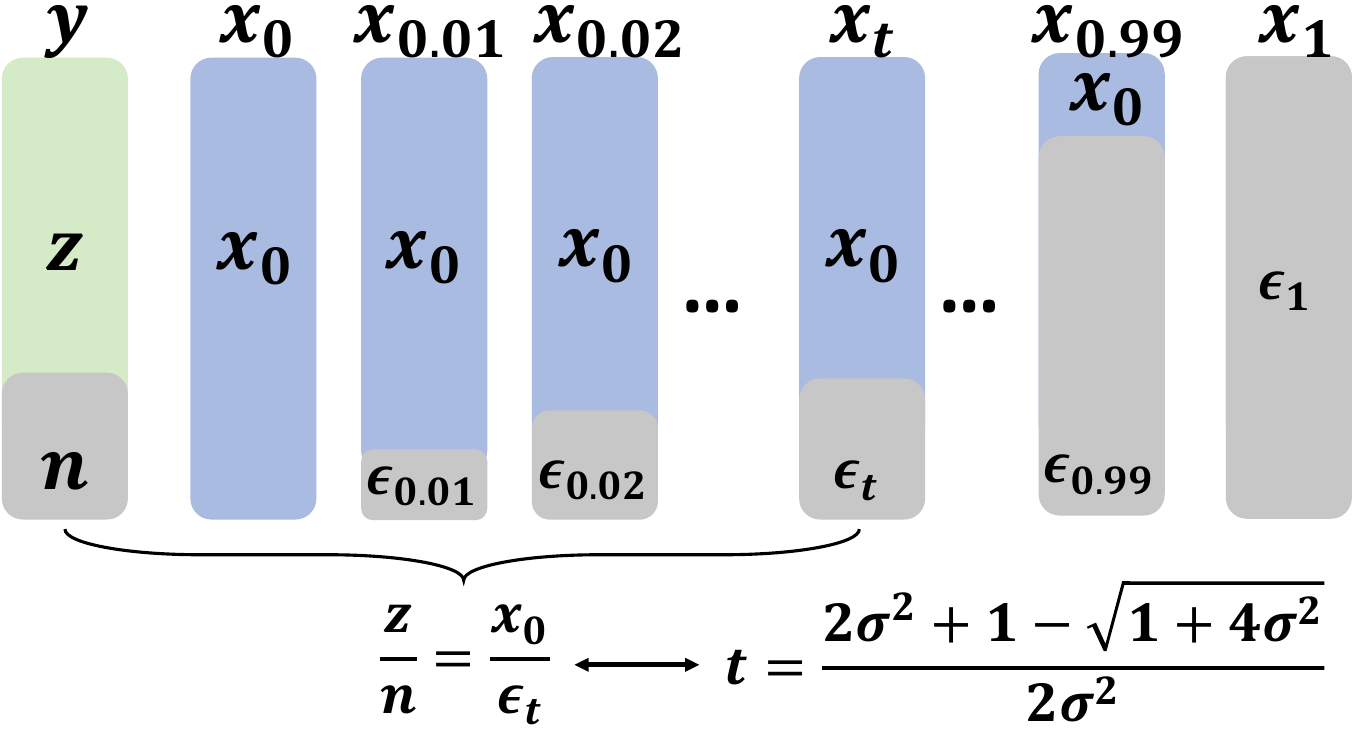}
    \caption{Illustration of the denoising timestep calculation.}
    \label{fig-t-calculation}
\end{figure}

Although DMs are predominantly used for generative tasks, their fundamental mechanism is grounded not in predicting future data states but in estimating the Gaussian noise that corrupts them. Rather than producing $\bm{x}_{t-\Delta t}$ directly from $\bm{x}_t$, the neural network within the DM is trained to infer the noise component added between these two states. This denoising-centric formulation underlies their formal designation as denoising diffusion probabilistic models\cite{ho2020denoising}, and has proven highly effective in reconstructing clean data across a range of applications. Importantly, the denoising process need not begin from pure Gaussian noise as $t = 1$. In practice, training is conducted over randomly sampled timesteps $t \in (0, 1]$, where the model learns to predict the noise corresponding to each level of corruption. As a result, DMs are inherently capable of initiating the reconstruction process from any intermediate noisy observation $\bm{x}_t$, making them flexible tools for tasks beyond generation.

From the perspective of semantic communication, this flexibility necessitates careful consideration of how the denoising timestep $t$ should be selected. As described in \eqref{ddm-forward}, the relative contributions of signal and noise in $\bm{x}_t$ are directly controlled by $t$. Larger values of $t$ correspond to higher noise content and diminished signal presence, whereas smaller $t$ values imply a cleaner input. If the timestep is mismatched to the actual corruption level, such as using a large $t$ for lightly corrupted data, useful semantic features may be lost. Conversely, initiating denoising from a small $t$ in the presence of severe noise may yield suboptimal results due to insufficient correction. To address this, we introduce an SNR-based strategy to align the timestep $t$ with the statistical characteristics of the received signal. The rationale draws from an analogy to classical detection theory: just as signal recovery aims to extract the deterministic component $\bm{z}$ from a noisy observation $\bm{y}$, diffusion denoising seeks to recover $\bm{x}_0$ from $\bm{x}_t$. Therefore, as is shown in Fig.~\ref{fig-t-calculation}, we propose that the optimal denoising performance can be achieved when the energy ratio between signal and noise in $\bm{x}_t$ is matched to that in $\bm{y}$, thereby guiding the selection of $t$ based on the estimated SNR.
The $t$ can be calculated as Theorem~\ref{theorem-1} and Remark~\ref{remark-2}.
\begin{theorem}\label{theorem-1}
    Assuming the $\mathbb{E}\left[\|\bm{x}_{0}\|^{2}\right]=\gamma$, for any given received noisy latent feature map $\bm{y}$, and noise density $\sigma^{2}$ of noisy channel, when 
    \begin{align}
    t = \frac{2+\phi-\sqrt{\phi^2+4\phi}}{2},\label{t-calculation-1}
    \end{align} 
    where $\phi=\frac{\mathbb{E}[\|y\|_{2}^{2}]-\sigma^2}{\gamma\sigma^2}$. The following equation can be obtained.
    \begin{align}
        \frac{\mathbb{E}\left[\|\eta\bm{z}\|_{2}^{2}\right]}{\mathbb{E}\left[\|\bm{n}\|_{2}^{2}\right]}=\frac{\mathbb{E}\left[\|(1-t)\bm{x}_{0}\|^{2}\right]}{\mathbb{E}\left[\|\sqrt{t}\bm{\epsilon}\|_{2}^{2}\right]}.\label{t-calculation}
    \end{align}
\end{theorem}
\begin{proof}
Because $\bm{z}$ and $\bm{n}$ are uncorrelated, thus $\mathbb{E}[\|\bm{y}\|_{2}^{2}]= \mathbb{E}[\|\eta\bm{z}+\bm{n}\|_{2}^{2}]=\mathbb{E}\!\left[\lVert\eta\bm{z}\rVert^{2}\right]+\sigma^{2}$. Therefore, the following equation can be obtained.
\begin{align}
    \mathrm{SNR}_{\mathrm{obs}}=\frac{\mathbb{E}\!\left[\|\eta\bm{z}\|_{2}^{2}\right]}{\sigma^{2}}
      =\frac{\mathbb{E}[\|\bm{y}\|_{2}^{2}]-\sigma^{2}}{\sigma^{2}}\label{snr-trans}
\end{align}
The purpose of transforming \eqref{snr-trans} is because, for the receiver, it cannot know the $\|\eta\bm{z}\|_{2}^{2}$ for a specific $\bm{z}$, since it doesn't the $\bm{z}$, but can only obtain the power of $\|\bm{y}\|_{2}^{2}$ and $\|\bm{n}\|_{2}^{2}$ through measurement. Therefore, it needs to calculate SNR through the above transformation. By introducing $\phi\triangleq \mathrm{SNR}_{\mathrm{obs}}/\gamma >0$, the matching condition of \eqref{t-calculation} can be obtained as follows.
\begin{align}
    (1-t)^{2}=\phi t.
\end{align}
Expanding yields the following equation.
\begin{align}
    t^{2}-(2+\phi)t+1=0.
\end{align}
The discriminant is $(2+\phi)^{2}-4=\phi^{2}+4\phi>0,$ ensuring two real roots as follows.
\begin{align}
    t_{\pm}=\frac{2+\phi\pm\sqrt{\phi^{2}+4\phi}}{2}.
\end{align}
Because $\phi>0$, the ``plus" root satisfies $t_{+}>1$ and violates the physical constraint $t<1$; consequently it is rejected.  The “minus’’ root can be obtained as follows.
\begin{align}
    t^{\star}=t_{-}=\frac{2+\phi-\sqrt{\phi^{2}+4\phi}}{2}.
\end{align}
\end{proof}

\begin{remark}\label{remark-2-1}
    According to Theorem~\ref{theorem-1}, consider the mapping $g(t;\phi)\triangleq (1-t)^{2}-\phi t,\qquad t\in[0,1].$ Then the following properties hold.
    (1) For every $\phi\ge0$ the equation $g(t;\phi)=0$ admits a unique solution; (2) The function $t^{\star}(\phi)$ is strictly decreasing in $\phi$; equivalently, $t^{\star}$ increases as $\mathrm{SNR}_{\mathrm{obs}}$ decreases; (3) When $SNR\rightarrow 0$, $t=1$ which equals to denoise from a total noise according to \eqref{ddm-forward}, and when $SNR\rightarrow +\infty$, $t=0$ which means no denoise is needed.
\end{remark}
\begin{proof}
    Observing that $g(0;\phi)=1>0$ and $g(1;\phi)=-(\phi)\le 0$ for every $\phi\ge0$, and noting the equation as follows.
    \begin{align}
        &\frac{\partial g}{\partial t}=2(t-1)-\phi\\ 
        &\frac{\partial^{2} g}{\partial t^{2}}=2>0, 
    \end{align}
    the $g(\cdot;\phi)$ is strictly convex on $[0,1]$.  A strictly convex function that changes sign exactly once on a closed interval possesses one and only one root; hence, the solution exists and is unique in $(0,1]$.  Solving $g(t;\phi)=0$ yields the explicit closed‐form root displayed above, whose radicand $\phi^{2}+4\phi=\phi(\phi+4)$ is non-negative for all $\phi\ge0$; consequently the square root is real and non-negative, ensuring $t^{\star}$ is well defined.

    To establish monotonicity, differentiate $t^{\star}$ with respect to $\phi$ is as follows.
    \begin{align}
        \frac{d t^{\star}}{d\phi}
          =\frac{1-\frac{1}{2}\bigl(2\phi+4\bigr)/\sqrt{\phi^{2}+4\phi}}
                 {2}
          <0,\\
          \forall\phi>0,
    \end{align}
    because $2\phi+4>0$ and $\sqrt{\phi^{2}+4\phi}>\phi$.  Thus $t^{\star}$ decreases strictly with increasing $\phi$. Finally, the limiting values follow directly from L’Hospital’s rule applied to the closed‐form expression \cite{zorich2016mathematical}. As $\mathrm{SNR}_{\mathrm{obs}}\rightarrow\infty$ we have $\phi\rightarrow\infty$ and therefore $t^{\star}\sim 1/\phi\rightarrow 0$, indicating that an infinitely clean observation requires no diffusion denoising.  Conversely, when $\mathrm{SNR}_{\mathrm{obs}}\rightarrow 0$ we obtain $\phi\rightarrow0$ and $t^{\star}\rightarrow 1$, meaning the denoiser must start from pure noise because the observation contains no discernible signal component.
\end{proof}

\begin{remark}\label{remark-2}
    In the training procedure, since the purpose of denoise is to solve $\bm{z}$, thus if the channel fading factor $\eta$ can be obtained through some channel estimation methods, the $\bm{x}_{0}$ is set equal to $\bm{z}$. Therefore, the $t$ can be simplified as follows.
    \begin{align}
        t = \frac{2\sigma^2+\eta-\sqrt{\eta^2+4\eta^2\sigma^2}}{2\sigma^2}.\label{calculate-t}
    \end{align}
\end{remark}
The absence of an explicit $\bm{z}$ term in Remark~\ref{remark-2} follows directly from the standard training protocol of diffusion models: during pre-training, the latent signal energy is normalised to a fixed range, which is equivalent to assuming a constant transmit power. Under this normalisation, the instantaneous SNR is governed solely by the additive-noise variance $\sigma^{2}$. Consequently, the closed-form expression in  Remark~\ref{remark-2} depends only on $\sigma^{2}$ and no longer involves $\bm{z}$, yielding a simplified yet general mapping between the channel noise level and the optimal denoising timestep.

\begin{figure}
    \centering
    \includegraphics[width=1\linewidth]{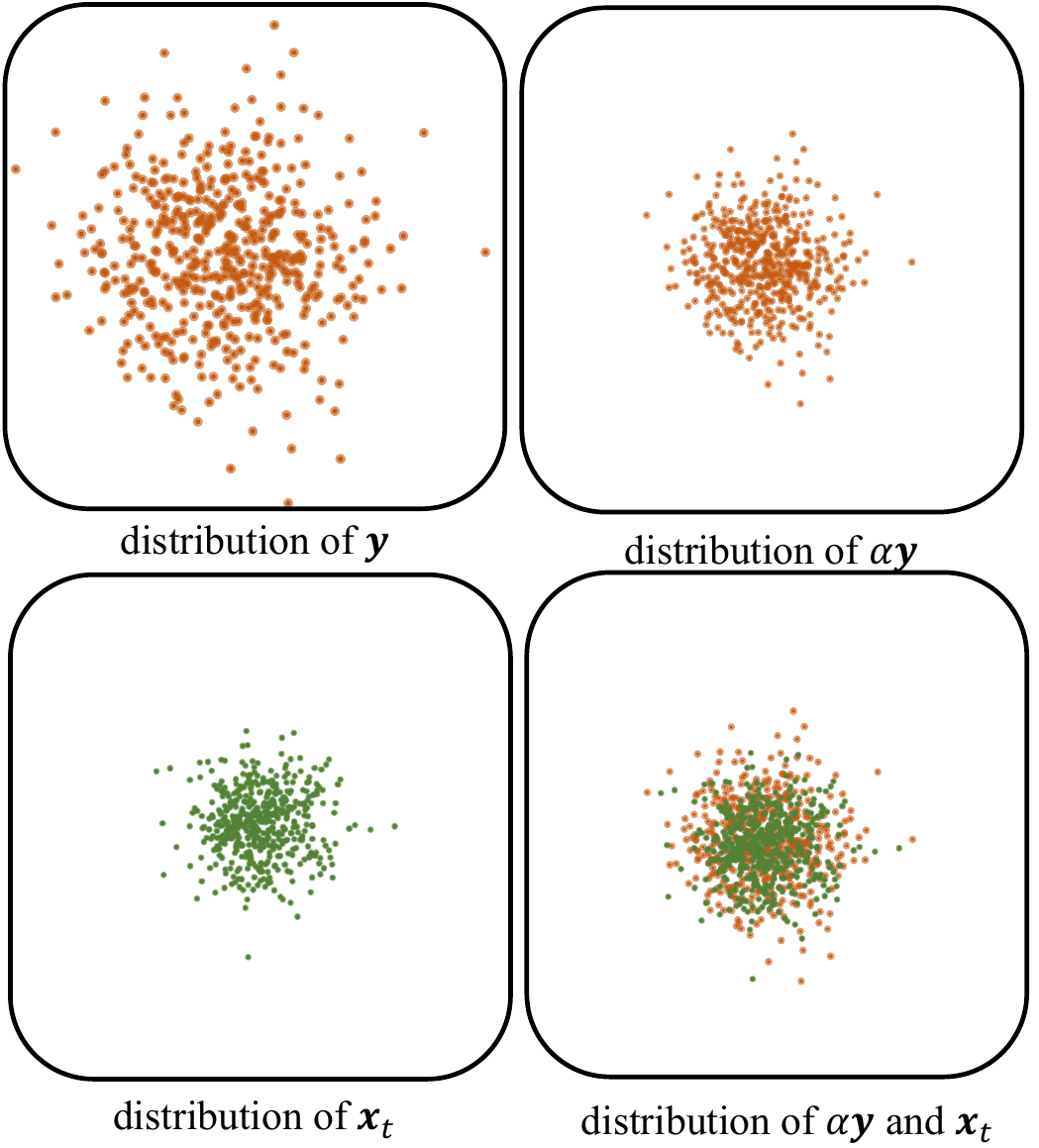}
    \caption{Illustration of the scaling method.}
    \label{fig-alpha-calculation}
    \vspace{-12pt}
\end{figure}
Theorem~\ref{theorem-1} guarantees the existence of a timestep $t\in(0,1]$ for which the signal–noise ratio embedded in the diffusion variable $\bm{x}_{t}$ exactly matches that of the received vector $\bm{y}$.  However, during the training procedure, the DM is exposed only to pairs $(\bm{x}_{0},\bm{x}_{t})$ generated via the forward process in \eqref{ddm-forward}. Consequently, the network learns the statistics of $\bm{x}_{t}$ rather than those of the true channel output.  If $\bm{y}$ is fed directly into the DM, the mismatch in input distribution leads to an OOD condition and a marked degradation in denoising accuracy, as is shown in Fig.~\ref{fig-alpha-calculation}. In wireless reception, this discrepancy can be reconciled by a simple power-normalisation step. Because both $\bm{x}_{t}$ and $\bm{y}$ are linear combinations of a deterministic term $\bm{z}$ and zero-mean Gaussian noise, their distributional forms are identical up to second-order moments.  Letting $\bm{x}_{0}=\bm{z}$, we equalise the average energies of $\bm{y}$ and $\bm{x}_{t}$ through a scalar factor $\alpha$, i.e., $\bm{y}'=\alpha\bm{y}$.  Selecting $\alpha$ according to Theorem 2 aligns the second-order statistics of $\bm{y}'$ with those of $\bm{x}_{t}$, thereby ensuring that the pretrained DM operates within its learned distribution.  This linear scaling not only restores denoising performance but also guarantees that the recovered latent vector coincides with the desired signal component $\bm{z}$ without further post-processing.

\begin{theorem}\label{theorem-2}
   Define the scaled observation $\tilde{\bm y}=\alpha\bm y$ with scalar $\alpha>0$.  The equality
\begin{align*}
    \mathbb E\!\bigl[\lVert\tilde{\bm y}\rVert^{2}\bigr]
     =\mathbb E\!\bigl[\lVert\bm x_{t}\rVert^{2}\bigr],
\end{align*}
holds if and only if
\begin{align}
    \;
\alpha
      =\sqrt{\frac{(1-t)^{2}(\mathbb{E}[\|\bm{x}_{0}\|_{2}^{2}])+t}{\mathbb{E}[\|\bm{y}\|_{2}^{2}]}}.
\end{align}
Moreover, $0<\alpha\le1$; $\alpha$ decreases strictly with $t$\footnote{Since $\bm{x}_{0}$ corresponds to the label in the DM training process, its distribution can be obtained.}.
\end{theorem}
\begin{proof}
The expected value of $\mathbb{E}\left[\|\alpha \bm{y}\|^{2}\right]$ is as follows.
    \begin{align*}
        \mathbb{E}\left[\|\alpha \bm{y}\|^{2}\right]&=\mathbb{E}\left[\|\alpha (\eta\bm{z}+\bm{n})\|^{2}\right],\\
        &=\alpha^2 \left(\mathbb{E}\left[\| \bm{s}\|^{2}\right]+\mathbb{E}\left[\|\bm{n}\|^{2}\right]\right).
    \end{align*}
Moreover, the expected value of $\mathbb{E}\left[\|\bm{x}_{t}\|^{2}\right]$ is as follows.
\begin{align*}
    \mathbb{E}\left[\|\bm{x}_{t}\|^{2}\right]&=\mathbb{E}\left[\|(1-t)\bm{x}_{0}+\sqrt{t}\bm{\epsilon}\|^{2}\right],\\
    &=(1-t)^{2}\mathbb{E}\left[\|\bm{x}_{0}\|\right]+t\mathbb{E}\left[\|\bm{\epsilon}\|^{2}\right],\\
    &=(1-t)^{2}\mathbb{E}\left[\|\bm{x}_{0}\|\right] +t
\end{align*}
Then, by solving the following equation, the value of $\alpha$ can be obtained as Theorem~\ref{theorem-2}.
\begin{align*}
    \mathbb{E}\left[\|\bm{x}_{t}\|^{2}\right]= \mathbb{E}\left[\|\alpha \bm{y}\|^{2}\right].
\end{align*}
Taking the principal square root preserves positivity and furnishes the closed form for $\alpha$.  Positivity follows from $t\in(0,1)$ and $\mathbb{E}[\|\bm{y}\|_{2}^{2}]>\sigma^{2}$.  To bound $\alpha$ by unity, observe that $(1-t)^{2}(\mathbb{E}[\|\bm{y}\|_{2}^{2}]-\sigma^{2})+t\le \mathbb{E}[\|\bm{y}\|_{2}^{2}]$, because $(1-t)^{2}\le1-t$ on $(0,1)$ and $t<1$. Strict decrease in $t$ is evident from the negative derivative of the numerator with respect to $t$.
\end{proof}

A direct consequence of Theorem~\ref{theorem-1} and Theorem~\ref{theorem-2} is that the properly rescaled observation $\tilde{\bm y}=\alpha\bm y$ attains the same second-order distribution as the synthetic sample $\bm x_{t}$ employed during diffusion training.  Because the clean component of the channel output equals the training data pairs as $\bm{z}=\bm{x}_{0}$, the only mismatch between the two random vectors is the pair of linear coefficients that multiply the signal term and the additive Gaussian term, respectively.  Selecting the step size $t$ via Theorem~\ref{theorem-1} ensures that the SNR satisfies $(1-t)^{2}\gamma/t=\gamma/\sigma^{2}$; the scaling factor derived in Theorem~\ref{theorem-2} ensure the $\mathbb{E}\left[\alpha\bm{y}\right]=\mathbb{E}\left[(1-t)\bm x_{0}+\sqrt t\bm{\epsilon}\right]$ with $\bm{\epsilon}\sim\mathcal N(\bm 0,\bm I)$. Therefore, the rescaled received vector replicates exactly the affine stochastic structure of $\bm x_{t}$.  Although $\gamma=\mathbb E[\lVert\bm x_{0}\rVert^{2}]$ appears in both theorems, it is known at training time because it can be calculated through the whole training data.

\begin{figure*}
    \centering
    \includegraphics[width=1\linewidth]{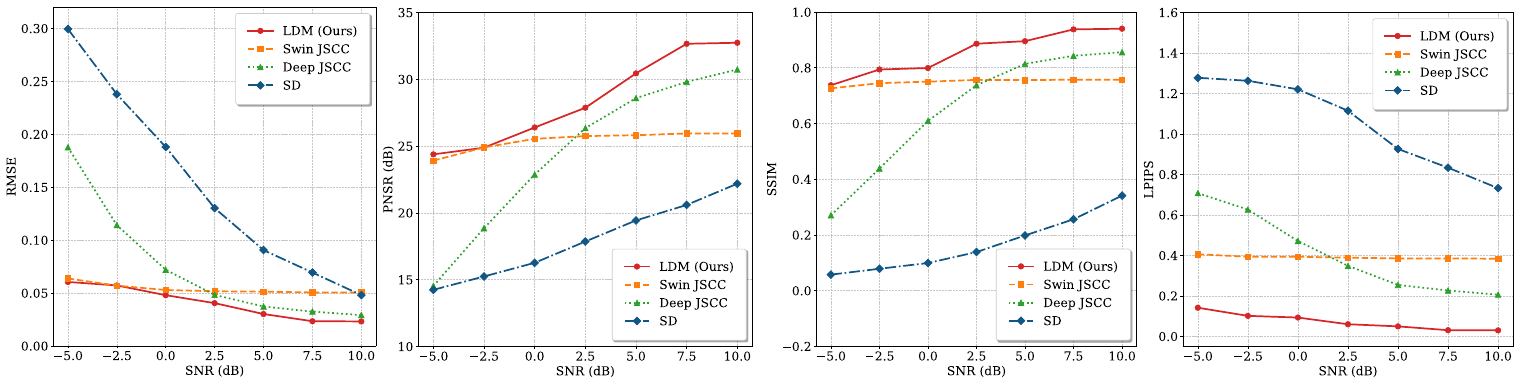}
    \caption{Performance evaluation of different methods on various SNRs in the CelebA-HQ datasets.}
    \label{fig-celeba}
\end{figure*}

\begin{figure*}[t]
    \centering
    \includegraphics[width=0.8\linewidth]{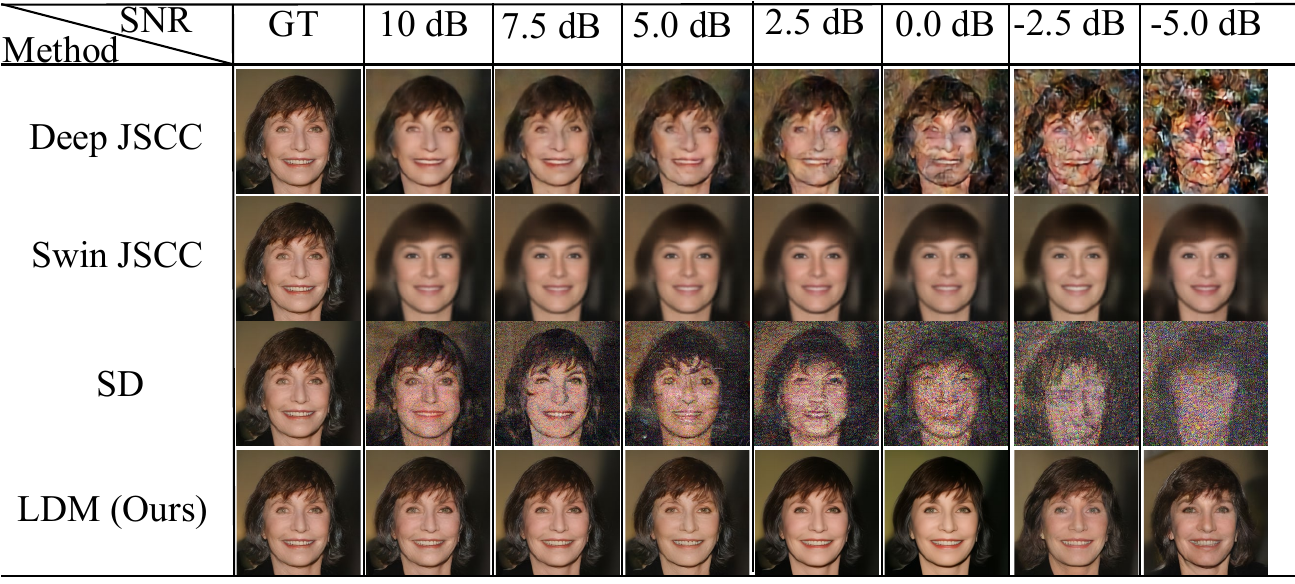}
    \caption{Comparison of image transmission performance at different SNR.}
    \label{fig-demo-celeba}
\end{figure*}
\subsection{LDM Based Sematic Communciation Framework}
To enable robust semantic communication over noisy channels, we propose an enhanced receiver architecture that integrates a pretrained LDM for adaptive denoising within the latent space, as shown in Fig.~\ref{fig-system}. Building upon the system model described earlier, the key novelty lies in the use of a generative denoising process that is aware of SNR and dynamically adapts to the characteristics of the received latent representation.

Given the received latent vector $\tilde{\bm{z}} = \eta\bm{z} + \bm{n}$, where $\bm{z}$ denotes the transmitted latent semantic representation and $\bm{n} \sim \mathcal{N}(0, \sigma^2 \bm{I})$ represents channel noise, the receiver first estimates the instantaneous SNR. This estimate is used to analytically determine a denoising timestep $t \in (0,1]$ according to Theorem~\ref{theorem-1}, which ensures that the signal–noise energy ratio in the diffusion variable matches that of the channel output. Simultaneously, a linear scaling factor $\alpha$ is computed via Theorem~\ref{theorem-2} to adjust the received vector's magnitude, aligning its distribution with that of the diffusion model's training inputs. This transformation yields the input to the LDM as $\bm{x}_{t}=\alpha \bm{y}$
The LDM then performs reverse diffusion starting from $\alpha\bm{\bm{y}}$, using the analytically determined timestep $t$ from Theorem~\ref{theorem-1}. Different from conventional denoisers, which are often retrained or fine-tuned under new noise conditions, the LDM leverages its learned generative prior to adapt across a wide range of channel perturbations without modification. The reverse process can be executed with a variable number of iterations, depending on the SNR: fewer steps under high-quality channels and more when the received signal is severely corrupted. In particular, the framework supports single-step reverse inference using the approximated expression in \eqref{ddm-reverse}, which enables low-latency operation when appropriate. Once denoised, the output $\hat{\bm{z}}$ is forwarded to the decoder, which reconstructs the semantic data. Since the LDM operates in latent space, its computational burden is low compared to pixel- or token-level generative models. Furthermore, the use of pretrained LDMs decouples the denoising mechanism from the encoder–decoder training, enabling modular updates. That is, more powerful LDMs—trained on large-scale data or advanced architectures—can be integrated directly, enhancing performance without altering the underlying semantic transceiver structure. This modular, SNR-aware framework allows the semantic communication system to maintain robustness against varying channel conditions while retaining the scalability and efficiency of a latent-space representation. It overcomes the limitations of conventional discriminative models, which suffer from generalization gaps when encountering unseen noise levels or out-of-distribution inputs, by grounding the denoising process in the theoretical structure of stochastic differential equations and distribution alignment.

\section{Experiment Results}
\begin{figure*}
    \centering
    \includegraphics[width=0.8\linewidth]{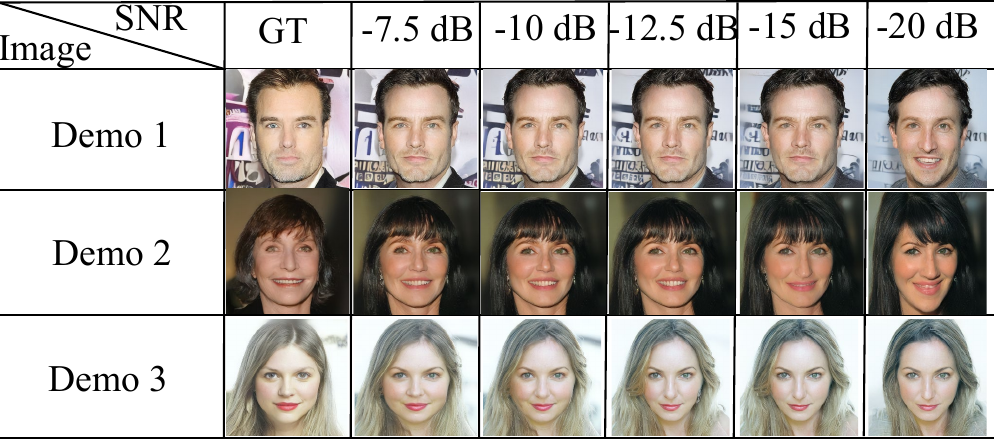}
    \caption{The image transmission performance of the proposed method on low SNRs.}
    \label{fig-low-snr}
\end{figure*}
\subsection{Datasets and Evaluation Metrics}
To evaluate the effectiveness of the proposed LDM-based semantic communication framework, we conducted extensive experiments on the CelebA-HQ dataset, a large-scale, high-resolution facial image corpus widely used for image generation and reconstruction tasks. As baseline comparisons, we selected two representative deep joint source-channel coding (JSCC) methods as follows. \textbf{Deep JSCC} \cite{bourtsoulatze2019deep}, a classical convolutional architecture for image transmission, and \textbf{Swin JSCC} \cite{yang2024swinjscc}, a state-of-the-art transformer-based semantic communication model known for its superior performance on vision tasks. In addition, we included \textbf{stable diffusion (SD)} \cite{LDM} as a generative baseline to assess the denoising capability of pretrained diffusion models in a semantic communication context. Specifically, the VAE in SD was utilized to perform joint source-channel encoding, analogous to the structure of the proposed LDM-based architecture. The corrupted latent features resulting from channel transmission were used as conditional inputs to the SD model, accompanied by a fixed textual prompt: “generate a clean and noise-free latent feature based on the content of this latent feature.” This prompt guided the model to perform direct semantic denoising using its pretrained generative prior. For our framework, we employed a pretrained latent diffusion model trained jointly on the CelebA-HQ \cite{karras2017progressive} and ImageNet \cite{deng2009imagenet} datasets, without any task-specific fine-tuning or post-training being applied, thereby highlighting the generalizability and zero-shot capability of the proposed architecture. To comprehensively assess reconstruction quality, we adopted both pixel-level and semantic-level evaluation metrics. Quantitatively, root mean squared error (RMSE) was used to measure pixel-wise distortion. Additionally, we evaluated semantic preservation and perceptual quality using peak signal-to-noise ratio (PSNR), structural similarity index measure (SSIM) \cite{1284395}, and learned perceptual image patch similarity (LPIPS) \cite{zhang2018unreasonable}. These metrics collectively reflect the system's fidelity in both low-level accuracy and high-level semantic consistency.

\subsection{Performance Comparison on CelebA-HQ}
The qualitative and quantitative performance of the proposed LDM-based semantic communication framework is illustrated in Fig.~\ref{fig-celeba}-\ref{fig-low-snr}. Fig.~\ref{fig-demo-celeba} presents visual comparisons of image reconstruction outcomes under varying SNR conditions. It is evident that the proposed method achieves superior restoration quality, even under severely degraded channel conditions. Notably, without any task-specific fine-tuning, the LDM-based approach consistently preserves semantic structure and fine-grained details across a broad SNR range. In contrast, the two baseline deep JSCC methods demonstrate significant degradation under low SNRs, manifesting in blurry textures and distorted geometries. While the Swin-JSCC architecture performs better at low SNRs and can generate relatively clean outputs, its reconstructions tend to be overly smooth, with visibly diminished detail fidelity. This smoothness suggests a loss of high-frequency semantic features, likely due to the reliance on discriminative learning mechanisms, which are vulnerable to OOD channel perturbations and lack the capacity for generative recovery of corrupted latent semantics. We also evaluated a SD-based approach, where the received noisy latent features were used as conditional inputs alongside a denoising prompt. Despite this prompt-driven guidance, the SD model failed to produce competitive results, primarily because it was never explicitly trained to perform conditional denoising in the latent space. As a result, the generated outputs exhibited structural inconsistencies and semantic drift. The effectiveness of the LDM-based approach is further validated in Fig.~\ref{fig-celeba}, which summarizes numerical performance across a wide range of SNRs. The proposed method outperforms all baselines in both pixel-level fidelity, which is measured by RMSE, and semantic similarity metrics, including PSNR, SSIM, and LPIPS. Of particular note is the system’s behavior in low-SNR regimes, such as SNR$< 0$dB, the degradation trends of all evaluation metrics flatten considerably for the LDM-based framework, indicating strong resilience to severe channel noise. Fig.~\ref{fig-low-snr} further substantiates this robustness by showcasing restored outputs at -10 dB and -20dB SNRs. Even at -10dB, the proposed system preserves most visual and structural semantics, and at -20dB, although fine-grained object details may be lost, high-level semantic attributes—such as facial orientation and gender identity—remain correctly reconstructed. These results highlight the LDM’s unique ability to capture and restore abstract semantic content under extreme channel impairments, demonstrating its advantage over both discriminative and conditional generative baselines.

\begin{figure*}[t]
    \centering
    \subfigure[Performance demonstration on OOD data of landscape.]
    {
       \centering
       \includegraphics[width=0.62\linewidth]{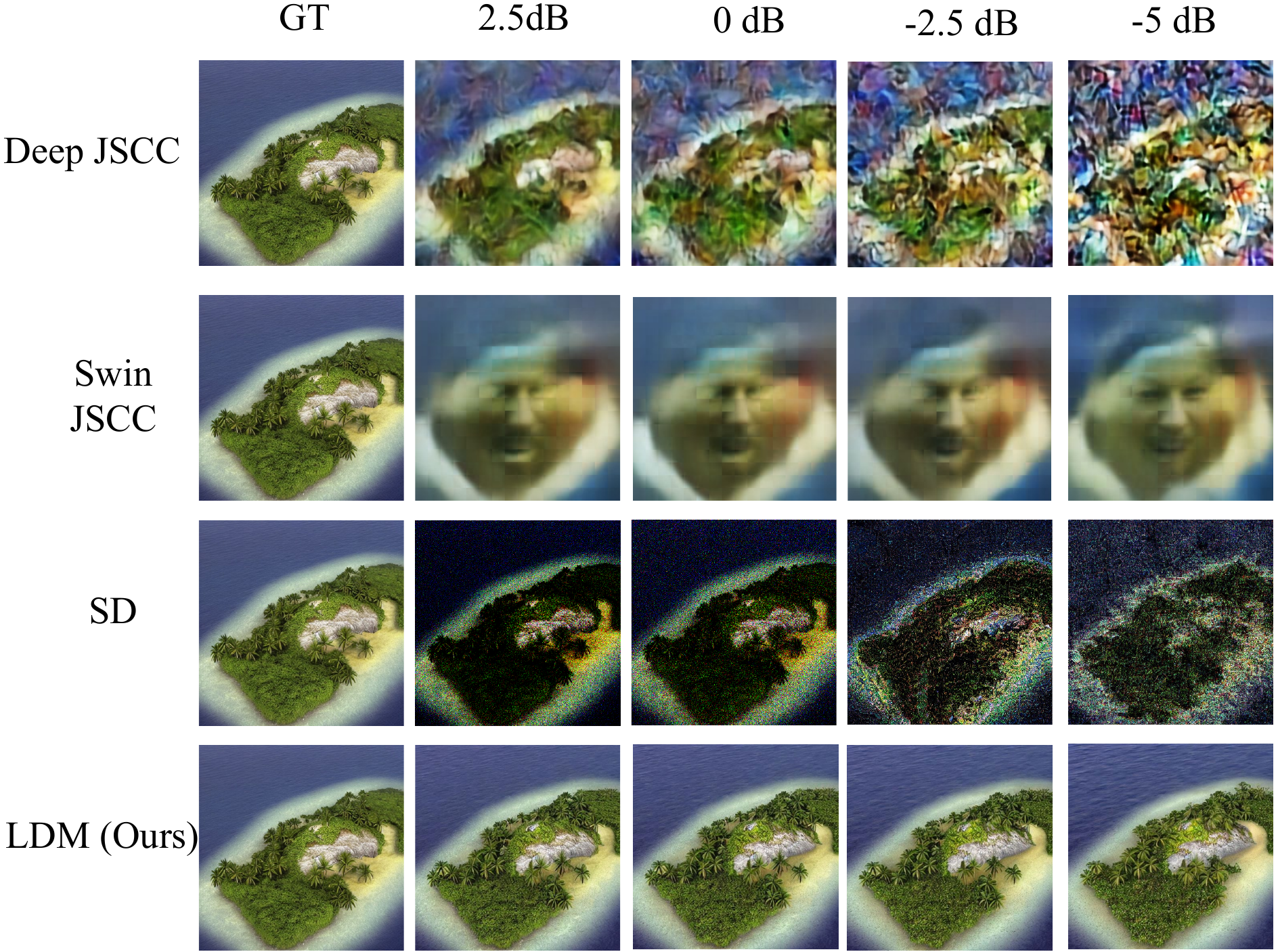}\label{fig-ood-island}
    }
    \subfigure[Performance demonstration on OOD data of animal.]
    {
       \centering
       \includegraphics[width=0.62\linewidth]{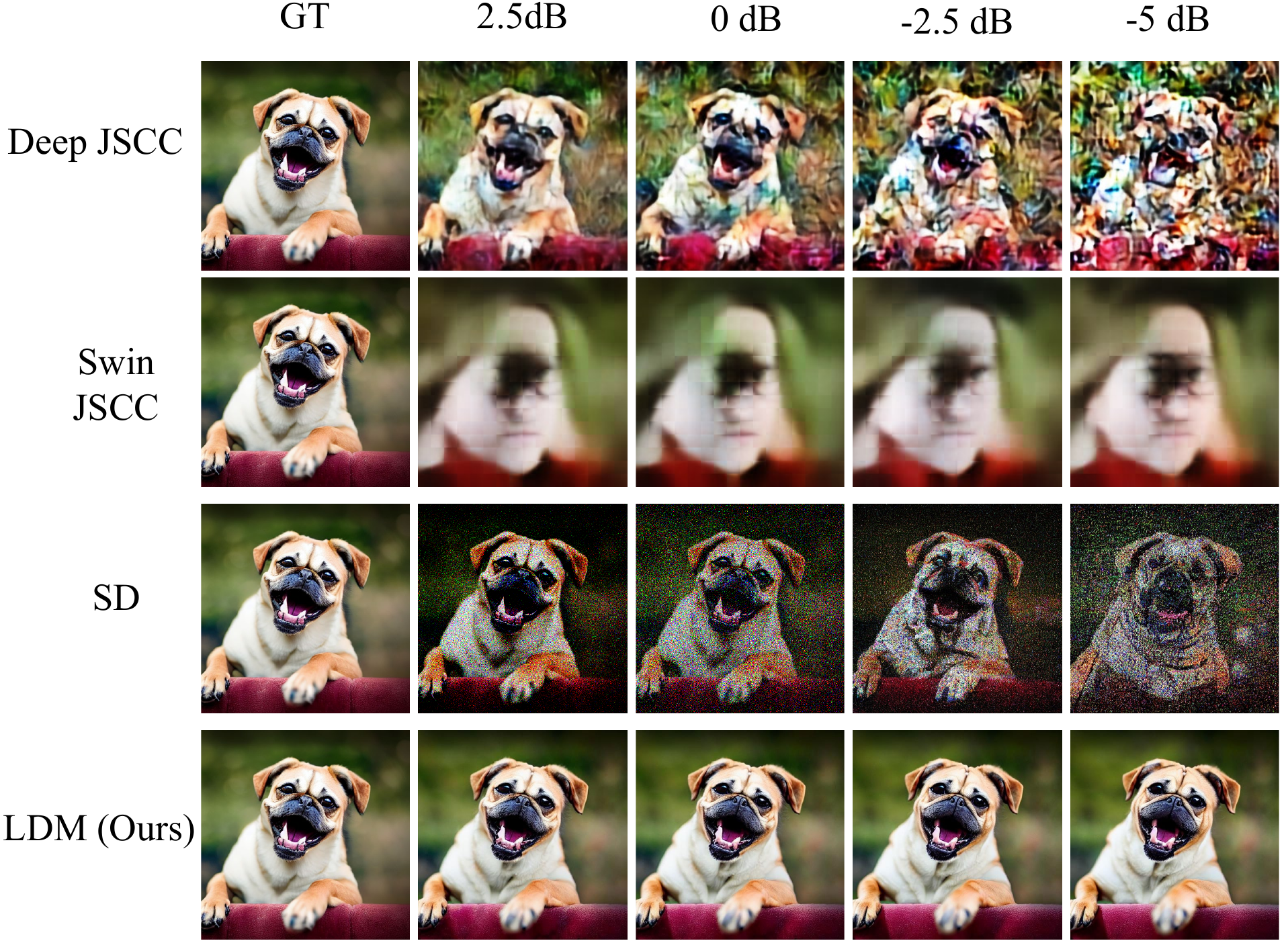}\label{fig-ood-dog}
    }
    \caption{Performance demonstration on OOD data.}
    \label{fig-ood}
\end{figure*}
\subsection{Performance Demonstration of OOD Data}
In addition to the semantic distortion caused by additive noise in the communication channel, semantic communication systems also face a more fundamental and often overlooked challenge: OOD generalization with respect to the transmitted data itself. While the majority of existing semantic communication research focuses on the robustness of latent feature recovery under noise, the problem of distributional shift between training data and real-world test data has received comparatively limited attention. In practice, semantic communication systems are often trained on specific datasets using supervised end-to-end learning pipelines, such as those adopted in classical Deep JSCC \cite{bourtsoulatze2019deep} and Swin-JSCC \cite{yang2024swinjscc} frameworks. Although these systems may achieve excellent reconstruction performance on in-distribution samples, their ability to generalize across diverse data distributions remains highly constrained. To empirically validate this limitation, we follow standard training protocols and train several semantic transceivers, including Deep JSCC and Swin JSCC, on the CelebA-HQ dataset—a high-resolution portrait dataset. We then evaluate their performance on semantically unrelated image categories, including natural landscapes and animal scenes. As illustrated in Fig.~\ref{fig-ood} Deep JSCC exhibits significant degradation under OOD conditions, particularly when the channel SNR drops below 0dB. The reconstructed outputs suffer from both structural distortions and semantic ambiguity. Although Swin JSCC demonstrates improved low-SNR performance, it exhibits signs of overfitting to the portrait domain: when applied to landscape or animal images, the outputs are dominated by unnatural human-like textures, regardless of the original content. This over-specialization severely limits its applicability to real-world communication scenarios, where data distributions are often nonstationary and diverse. Interestingly, we observe that the SD model, though not optimized for end-to-end transmission, is still capable of preserving coarse semantic information, such as global shapes and object silhouettes, across different content types. This behavior is attributed to its training on large-scale, heterogeneous datasets such as LAION and ImageNet, which enables the model to learn a generalizable representation of natural images. However, due to the lack of explicit training for channel noise suppression, the SD model struggles to recover detailed features or suppress semantic corruption introduced by the channel. 
\begin{figure*}
    \centering
    \includegraphics[width=1\linewidth]{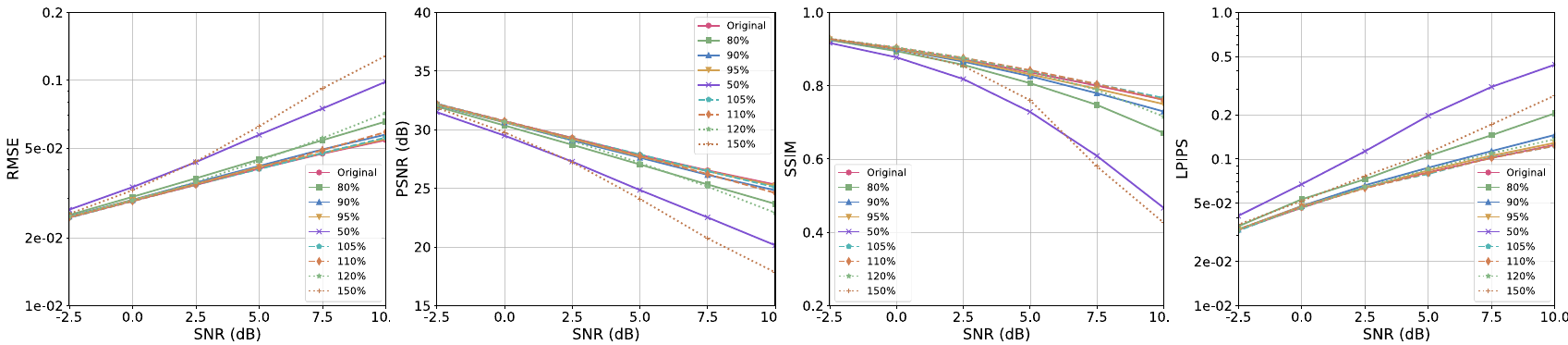}
    \caption{Performance evaluation on different $t$.}
    \label{fig-t}
\end{figure*}
\begin{figure*}
    \centering
    \includegraphics[width=1\linewidth]{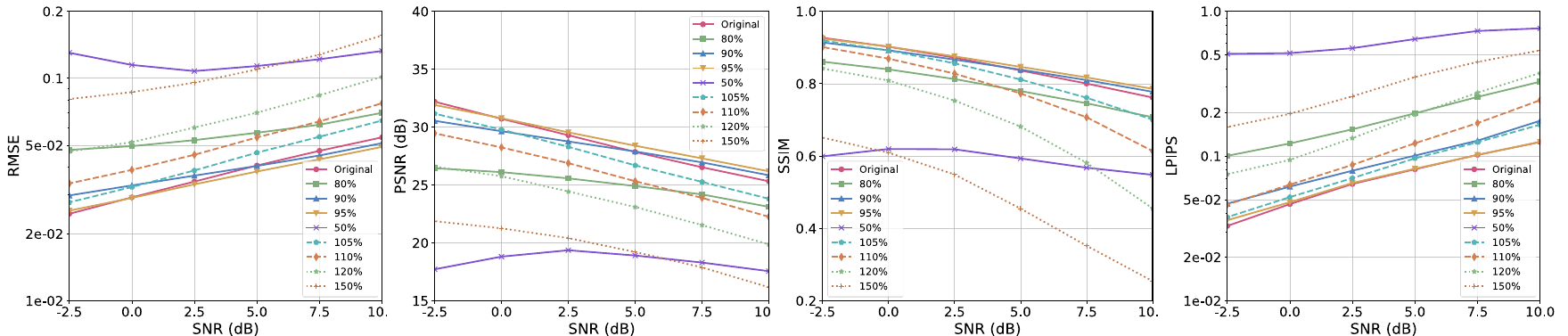}
    \caption{Performance evaluation on different $\alpha$.}
    \label{fig-alpha}
\end{figure*}



In contrast, the proposed LDM-based semantic communication framework directly addresses both noise robustness and data distribution generalization. By design, it decouples source-channel encoding from denoising and reconstruction, employing a pretrained LDM as a universal noise suppressor. A key advantage of our architecture is its modularity: the LDM component can be readily replaced with publicly available weights trained on large-scale datasets such as ImageNet or OpenImages. Without any task-specific fine-tuning, the system is capable of adapting to a wide range of visual domains. This makes it particularly well-suited for deployment in open-world communication environments where training data may not reflect future transmission requirements. Furthermore, because the LDM operates in the latent space of a VAE encoder, it preserves compactness and scalability while still benefiting from the rich generative priors of large-scale diffusion models. This plug-and-play compatibility with open-source generative backbones imbues the system with a self-evolutionary capacity—as more powerful generative models become available, they can be directly integrated into the communication pipeline to enhance performance without retraining the encoder or decoder. In summary, the proposed framework not only resolves the vulnerability of semantic communication systems to channel-induced feature corruption, but also significantly improves cross-domain generalization through its integration with generative artificial intelligence. The ability to support robust semantic reconstruction under both SNR degradation and distributional shift underscores the practical utility and scalability of our approach for next-generation communication systems.

\subsection{Numerical Proof of Theorem \ref{theorem-1} and \ref{theorem-2}}
To substantiate the theoretical claims made in Theorems \ref{theorem-1} and \ref{theorem-2} regarding the optimal denoising timestep $t$ and scaling factor $\alpha$, we perform a comprehensive numerical evaluation to assess their impact on end-to-end semantic reconstruction performance. Specifically, we aim to verify whether the closed-form expressions derived for $t$ and $\alpha$, which are analytically linked to the SNR and data distribution statistics, indeed correspond to optimal operational points in practice. In this experiment, we compare the performance of our LDM-based semantic communication framework under four different evaluation metrics: RMSE, peak PSNR, SSIM, and LPIPS. We systematically perturb the analytically computed values of $t$ and $\alpha$ by $\pm$5\%, $\pm$10\%, $\pm$20\%, and $\pm$50\% to simulate scenarios where the SNR is inaccurately estimated or where the parameter selection deviates from the ideal theoretical value. These perturbations are intended to reflect practical conditions, where perfect knowledge of channel statistics may not be available.

Fig.~\ref{fig-t} illustrates the impact of these deviations in $t$ on system performance. Across all metrics, the best results are consistently achieved when using the theoretically derived value of $t$. As the deviation increases, a clear monotonic degradation in performance is observed. This empirical behavior aligns with our theoretical expectation: the denoising performance of the diffusion model is highly sensitive to the matching between the true noise content and the chosen diffusion timestep. The further the assumed timestep diverges from the one that corresponds to the actual SNR, the more the model either overestimates or underestimates the noise level, leading to under-denoising or semantic oversmoothing, respectively. An additional and noteworthy observation is that when $t$ is perturbed within a small margin, such as $\pm$5\%, the system performance remains relatively stable across all indicators. This demonstrates the robustness of the proposed parameter selection mechanism, suggesting that our framework does not require highly precise SNR estimation to function effectively. Instead, a coarse approximation of the channel condition is sufficient for achieving near-optimal denoising behavior—an important practical advantage in dynamic or resource-constrained communication scenarios.

Fig.~\ref{fig-alpha} reports similar findings for the scaling factor $\alpha$. Again, the analytically computed value achieves the best or near-best performance across the board, validating the correctness of our theoretical derivation. The advantage of using the exact value is especially pronounced in low-SNR regimes, such as  SNR$< -2.5$dB, where even minor inaccuracies in scaling lead to severe mismatches between the input distribution of the denoiser and its training distribution. In contrast, the theoretically derived scaling factor enables optimal distribution alignment and preserves denoising fidelity. These results confirm the critical role of precise distribution matching in the success of generative denoising models under severe noise conditions. In summary, the experimental results provide strong numerical evidence supporting the correctness and effectiveness of the closed-form expressions for $t$ and $\alpha$ proposed in Theorems \ref{theorem-1} and \ref{theorem-2}. Moreover, the demonstrated performance robustness to small deviations further highlights the practicality and resilience of our approach in real-world communication systems where estimation errors are inevitable. This validates the proposed framework not only from a theoretical standpoint but also from a system design and deployment perspective.

\section{Conclusion}
In this work, we have proposed a novel semantic communication framework based on LDMs and have established a rigorous theoretical foundation grounded in stochastic differential equations to guide the denoising process. We have further derived closed-form solutions for the optimal diffusion timestep and input scaling factor, enabling robust semantic reconstruction without requiring model fine-tuning or retraining. By leveraging pretrained generative models and adapting them through analytical mappings to the wireless channel conditions, the proposed method has demonstrated strong generalization, noise resilience, and compatibility with diverse data distributions, making it well-suited for practical deployment in future wireless communication systems. In future work, we will explore the extension of our framework to multi-modal semantic communication tasks, incorporate adaptive diffusion control for real-time applications, and investigate joint training strategies to further optimize end-to-end performance under dynamic and multi-user network environments.

\bibliography{ref}
\bibliographystyle{IEEEtran}
\ifCLASSOPTIONcaptionsoff
  \newpage
\fi
\end{document}